\documentclass[11pt,reqno]{amsart}

\usepackage{amsmath,amsfonts,amssymb,commath}
\usepackage{graphicx,harvard}

\usepackage{multibib}

\newcommand{\Expect}{\mathbf E}

\newcommand{\Prob}{\mathbf P}

\newcommand{\argmax}{\arg\max}

\newtheorem{theorem}{Theorem}
[section]

\newtheorem{definition}[theorem]{Definition}
\newtheorem{lemma}[theorem]{Lemma}
\newtheorem{corollary}[theorem]{Corollary}

\newtheorem{example}[theorem]{Example}

\def\bigO{{\sf O}}

\begin{document}

\title[Iterative Weak Learnability and Multi-Class AdaBoost]{\bf Iterative Weak Learnability and Multi-Class AdaBoost}

\author{In-Koo Cho}
\address{Department of Economics, Emory University, Atlanta, GA 30322 USA}
\email{icho30@emory.edu}
\urladdr{https://sites.google.com/site/inkoocho}

\author{Jonathan Libgober}
\address{Department of Economics, University of Southern California, Los Angeles, CA 90089 USA}
\email{libgober@usc.edu}
\urladdr{http://www.jonlib.com/}

\date{\today}

\thanks{We are grateful for the hospitality and support from the University of Southern California.  We thank Shaowei Ke, Tim Roughgarden, and Rob Schapire for their helpful comments.  The first author gratefully acknowledges financial support from the National Science Foundation and the Korea Research Foundation.}

\begin{abstract}
  We construct an efficient recursive ensemble algorithm for the multi-class classification problem, inspired by SAMME (\citeasnoun{ZhuZouRossetandHastie09}).   We strengthen the weak learnability condition in \citeasnoun{ZhuZouRossetandHastie09} by requiring that the weak learnability condition holds for any subset of labels with at least two elements. This condition is simpler to check than many proposed alternatives
  (e.g., \citeasnoun{MukherjeeSchapire2013}).  As SAMME, our algorithm is reduced to the Adaptive Boosting algorithm (\citeasnoun{SchapireandFreund12}) if the number of labels is two, and can be motivated as a functional version of the steepest descending method to find an optimal solution.   In contrast to SAMME, our algorithm's final hypothesis converges to the correct label with probability 1. For any number of labels, the probability of misclassification vanishes exponentially as the training period increases.
The sum of the training error  and an additional term, that depends only on the sample size, bounds the generalization error of our algorithm as the Adaptive Boosting algorithm.\\[2pt]
\end{abstract}

\maketitle

\section{Introduction}
\label{Introduction}

While the Adaptive Boosting (AdaBoost) algorithm (\citeasnoun{SchapireandFreund12}) and its variants have been remarkably successful for binary classification problems, difficulties with generalizing to the multi-label case are well-known.\footnote{For a comprehensive survey, see \citeasnoun{FriedmanHestieandTibshirani00}.}
  Early approaches to this problem proposed dividing the multi-class problem into multiple two-class problems (e.g., \citeasnoun{SchapireandSinger99}). These approaches either do not inherit all of the desirable properties of AdaBoost or lack its intuitive descriptions. The extended algorithms are expected to be efficient: the algorithm must produce a forecast, which converges to the correct label in probability, and the probability of misclassification vanishes at an exponential rate. However, the extended algorithm is often substantially different from the original formula of \citeasnoun{SchapireandFreund12}) and is significantly more elaborate than the original formula. The aesthetic differences with AdaBoost are substantive.  First, it is less straightforward to analyze the properties of these algorithms when they are more elaborate. Second, the conditions required to ensure they work are often significantly more stringent.  

\citeasnoun{ZhuZouRossetandHastie09} developed a multi-class generalization of AdaBoost that maintains the same structure as the binary label case, called Stagewise Additive Modeling using a Multi-class Exponential loss function (SAMME). The algorithm has a statistical foundation as a functional version of the recursive steepest descending method to calculate the minimum of an exponential loss function.  SAMME is reduced to the original AdaBoost if the number of labels is two. 
\citeasnoun{ZhuZouRossetandHastie09} demonstrated through numerical simulations that if the set of weak hypothesis satisfies a version of the weak learnability condition, which extends the condition of \citeasnoun{SchapireandFreund12}, SAMME converges to the correct labels.

However, SAMME fails to generate a forecast, which converges to the correct label in probability. \citeasnoun{MukherjeeSchapire2013} showed that the weak learnability condition
of \citeasnoun{ZhuZouRossetandHastie09} is not adequate for constructing an efficient recursive ensemble algorithm. \citeasnoun{MukherjeeSchapire2013} instead imposed a condition called \emph{Edge-over-Random} (EOR), which involves checking a richer set of conditions.   Verifying EOR is not straightforward because we have to weigh the loss differently for every kind of misclassification. 

This paper proposes an alternative approach. We show that a slight modification of SAMME, combined with a slightly stronger version of the weak learnability of \citeasnoun{ZhuZouRossetandHastie09}, allows us to construct an efficient recursive ensemble algorithm. We strengthen the weak learnability condition by requiring that the same weak learnability continue to hold for any non-empty subset of labels with at least two elements. We call the stronger version of the weak learnability {\em iterative weak learnability}.   One might wonder whether the stronger version of weak learnability requires an exceedingly elaborate class of weak hypotheses. On the contrary, the set of single threshold classifiers satisfies iterative weak learnability, which is the set of the simplest non-trivial classifiers.

We depart from the conventional view of a boosting algorithm as a process to produce a final hypothesis, which is increasingly accurate as the training period increases. Instead, we consider the boosting algorithm as a process to eliminate inaccurate labels recursively, leaving the correct label in each round of elimination. If the set of labels has two elements, the two views are two sides of the same coin. As an example, let us consider the original AdaBoost with two labels. One can view the algorithm as a process to identify a correct label for each observation. Instead, we view the same algorithm as a process to eliminate an incorrect label. If the number of labels is larger than two, the second view is much easier to generalize than the conventional view on the boosting algorithm.

Our algorithm can be motivated as an implementation of Elimination by Aspects (EBA) (\citeasnoun{Tversky72}).  If we regard a label as an attribute, our exercise can be viewed as a procedure to identify the best possible attribute for the observation. The decision maker evaluates an attribute according to a linear function of different attributes after a sequence of observations and eliminates an attribute with the negative value of the linear function.   Our algorithm selects the attribute which survives the iterative elimination process.\footnote{We are grateful to Shaowei Ke for pointing out the possible link and the reference.}

We modify the SAMME algorithm to construct an ensemble algorithm to eliminate incorrect labels for each observation. By eliminating incorrect labels iteratively, the number of remaining labels becomes two, and then, our algorithm becomes the original AdaBoost algorithm to identify the correct label. We show that the probability of eliminating the correct label vanishes at an exponential rate as the training period increases, which implies that our algorithm accurately eliminates incorrect labels efficiently, in contrast to the original proposal.   One of the benefits of our approach is that we can analyze this algorithm following steps similar to the original binary label AdaBoost algorithm. Our algorithm produces a generalization error bound closely related to that of the original formulation.

Section \ref{Preliminaries} illustrates the problem.   Section \ref{Construction of Algorithm}  constructs the algorithm and states the main result.   Section \ref{Generalization Error} shows that the generalization error vanishes as the number of samples increases, proving that our algorithm is robust against the overfitting as the original AdaBoost is.
Section \ref{Conclusion} concludes the paper.

\section{Preliminaries}
\label{Preliminaries}

\subsection{Basics}

Given set $X$, let $\abs{X}$ be the number of elements in $X$. Let $A$ be the set of labels, and ${\mathcal P}$ be the set of observations, which is endowed with a probability distribution. A classifier is
\[
h : {\mathcal P}\rightarrow A
\]
which labels each observation. We often call $h$ a weak hypothesis (or simply, a hypothesis) in place of a classifier. Let ${\mathcal H}$ be the set of all feasible classifiers.

Let
\[
y : {\mathcal P}\rightarrow A
\]
be the correct classifier.
We say that $h$ perfectly classifies $y$ if
\[
\Prob\left( h(p)=y(p) \right)=1.
\]
We assume that a decision maker is endowed with ${\mathcal H}$, but $y$ is often significantly more complex than any element in ${\mathcal H}$. As a result, no $h\in {\mathcal H}$ can perfectly classify $y$. An important question is whether a decision maker who is endowed with primitive ${\mathcal H}$ can ``learn'' $y$. 

\begin{definition}
Let $\Gamma$ be the set of all classifiers, and $\tilde{\Gamma} \subset \Gamma$ denote a subset of classifiers. A \textit{statistical procedure} or \textit{algorithm} is an onto function
\[
\tau : \mathcal{D} \rightarrow \tilde{\Gamma},
\]
where $\mathcal{D}$ is a set of histories, $\mathcal{T}$ is the set of feasible algorithms.
\end{definition}

To incorporate learning, we need to allow a decision maker to observe data over time.
We consider a dynamic interaction, where the algorithm can update its output in response to new data. Suppose that time is discrete: $t=1,2,\ldots$ and $D_t$ is the data available at the beginning of time $t$. $\tau(D_t)$ is the classifier which algorithm $\tau$ generates from data $D_t$.   Let $\tau(D_t)(p)$ be the label assigned to observation $p$ by $\tau(D_t)$.

Let $T\ge 1$ be the training period, when the decision maker updates $\tau(D_t)$ in response to $D_t$. The final hypothesis is the output $\tau(D_T)$ at the end of the training period.

We focus on \emph{recursive ensemble algorithm} for its simplicity.

\begin{definition} \label{def:ensemble}
$\tau(D_K)$ is an ensemble of ${\mathcal H}$ if $\exists h_1,\ldots,h_K\in {\mathcal H}$ and $\alpha_1,\ldots,\alpha_K\ge 0$ such that $\forall p\in {\mathcal P}$
\[
\tau(D_K)(p)= \argmax_{\hspace{-7mm} a} \sum_{k=1}^{K} \alpha_k \mathbf{1}[a=h_{k}(p)]
\]
$\tau$ is a recursive algorithm, if $\tau(D_k)$ is determined as a function of  $\tau(D_{k-1})$ and events in time $k$.
\end{definition}

We require that an algorithm produce an accurate classifier as the final hypothesis and achieve the desired accuracy level with a reasonable amount of data.

\begin{definition}
Algorithm $\tau$ is consistent if
\[
\lim_{K\rightarrow\infty}\Prob\left(\tau(D_K) \ne y(p) \right)=0.
\]
Algorithm $\tau$ is efficient if $\exists \rho>0, K>0$ such that
\[
\Prob\left( \exists k\ge K, \tau(D_k)(p)\ne y(p) \right)\le e^{-\rho K}.
\]
\end{definition}

\subsection{Adaptive Boosting Algorithm}

Let us illustrate the Adaptive Boosting (AdaBoost) algorithm by \citeasnoun{SchapireandFreund12}. Instead of the original formula, we follow the framework of SAMME (\citeasnoun{ZhuZouRossetandHastie09}). Let us denote the algorithm as $\tau_A$. 

Initialize $d_1(p)=1/\abs{{\mathcal P}}$ as the uniform distribution over ${\mathcal P}$. For $k\ge 1$, $D_k$ is the history at the beginning of time $k$. Iterate the following steps.

\medskip

\begin{enumerate} 

\item (calculation of forecasting error)
Choose $h_k\in {\mathcal H}$.
\[
\epsilon_k=\sum_{p} d_k(p) {\mathbb I}\left( h_k(p)\ne y(p)\right)
\]
If $\epsilon_k=0$, then stop the iteration and set $\tau(D_k)=h_k$. If not, continue.

\item (calculating the weight)\footnote{The following formula is from equation (11) on page 353 of \citeasnoun{ZhuZouRossetandHastie09}, which becomes the coefficient of AdaBoost if $\abs{A}=2$.   This formula differs slightly from (c) on page 351.}
    
\begin{equation}
\alpha_k=\frac{(\abs{A}-1)^2}{\abs{A}}
\log\left( \frac{(\abs{A}-1)(1-\epsilon_k)}{\epsilon_k}\right)
\label{eq: SAMME alpha}
\end{equation}

\item (preparing for the next round)
\[
w_{k+1}(p)=\begin{cases}
d_k(p)\exp\left( -(\abs{A}-1)\alpha_k \right) & \text{if } \ h_k(p)= y(p) \\
d_k(p)\exp\left( \alpha_k\right) & \text{if } \ h_k(p)\ne y(p).
\end{cases}
\]
\[
d_{k+1}(p)=\frac{w_{k+1}(p)}{\sum_{p'}w_{k+1}(p')}.
\]

\item Reset $k+1$ as $k$ and start the first step of selecting an optimal weak hypothesis.

\item (final hypothesis)
\[
\tau_A(D_t)(p)=\arg\max_{a\in A}\sum_{s=1}^t\alpha_s{\mathbb I}(h_s(p)=a) \qquad\forall p\in {\mathcal P}
\]

\end{enumerate}

\medskip
If $\abs{A}=2$, $\tau_A$ is the Adaptive Boosting (AdaBoost) algorithm of \citeasnoun{SchapireandFreund12}. 

The description of the algorithm is incomplete because if $\alpha_k<0$, the algorithm is not well defined.

\begin{definition}  
  Let $A$ be the set of labels.  $\frac{1}{\abs{A}}$ random guess, or simply a random guess, is to predict each label with probability $\frac{1}{\abs{A}}$.    
\end{definition}

$\tau_A$ is meaningful only if $h_k\in {\mathcal H}$ performs better than a random guess.
\begin{equation}
\sum_{p \in {\mathcal P}} \left(\mathbb{I}[h_k(p)=y(p)] -\mathbb{I}[h_k(p) \neq y(p)] \right) d(p) \geq \frac{2-\abs{A}}{\abs{A}} + \rho,
\label{eq: meaningful}
\end{equation}
which implies that $\alpha_k>0$.   Note that $\frac{2-\abs{A}}{\abs{A}}$ is the score from the random guess.

The weak learnability implies the existence of such $h_k$. We state the condition according to \citeasnoun{ZhuZouRossetandHastie09}.

\begin{definition} \label{def: weaklearn}
A hypothesis class $\mathcal{H}$ is weakly learnable if, for every probability distribution $d$ over observations $p \in {\mathcal P}$ and labels $y(p)$, we have: 
\begin{equation} 
\min_{d \in \Delta({\mathcal P})} \max_{h \in \mathcal{H}} \sum_{p \in {\mathcal P}} \left(\mathbb{I}[h(p)=y(p)] -\mathbb{I}[h(p) \neq y(p)] \right) d(p) \geq \frac{2-\abs{A}}{\abs{A}} + \rho,
\label{eq: weak learnability}
\end{equation}
for some $\rho > 0$. 
\end{definition}

The weak learnability is sufficient for the efficiency of $\tau_A$ if $\abs{A}=2$.

\begin{theorem} If ${\mathcal H}$ is weakly learnable and $\abs{A}=2$, then $\tau_A$ is an efficient algorithm.
\end{theorem}

\begin{proof}
See \citeasnoun{SchapireandFreund12}.
\end{proof}

\citeasnoun{MukherjeeSchapire2013} pointed out that the notion of weak learnability of 
\citeasnoun{ZhuZouRossetandHastie09} is too weak to guarantee the convergence of the algorithm to the correct label if $\abs{A}>2$.

Let us examine an example in \citeasnoun{MukherjeeSchapire2013}.

\begin{example} \label{ex: MS13}
${\mathcal P}=\{ a, b\}$, $A=\{1,2,3\}$, $y(a)=1, y(b)=2$.
${\mathcal H}=\{ h_1,h_2\}$ where $h_1(a)=h_1(b)=1$ and $h_2(a)=h_2(b)=2$.
\end{example}

Since $\abs{A}=3$, the weak learnability is satisfied if the performance of a weak hypothesis is strictly better than
\[
\frac{2-3}{3}=-\frac{1}{3}.
\]
Let $d=(d(a),d(b))$ be a probability distribution over ${\mathcal P}$.
The performance score of ${\mathcal H}$ is
\[
\max\left[ 1-2d(a), -1+2d(a) \right] \ge 0
\]
where the equality holds if $d(a)=0.5$. Thus, ${\mathcal H}$ is weakly learnable. Any convex combination of $h_1$ and $h_2$ in ${\mathcal H}$
assigns the same level to two observations. Since $y$ assigns a different label on a different observation, SAMME fails to produce a consistent forecast, although ${\mathcal H}$ satisfies the weak learnability.

Example \ref{ex: MS13} reveals two issues we must address to extend AdaBoost from binary to multiple labels. First, to construct a boosting algorithm, it is necessary to strengthen the notion of weak learnability. Second, the set of weak hypotheses ${\mathcal H}$ cannot be too restrictive. The issue identified in the previous example is that the requirement of outperforming a random guess by $\rho$ becomes less restrictive when the number of labels gets larger. Hence some distributions can satisfy \ref{eq: weak learnability} for many labels but can fail for a small number of labels. We propose to consider the minimal departure necessary to rule this out.  

\subsection{Iterative Weak Learnability}

We strengthen the weak learnability by requiring that \eqref{eq: weak learnability} holds for any subset of $A$ with more than a single element. Let $A_i\subset A$ be a non-empty subset of $A$ with $\abs{A_i}\ge 2$. Define
\[
{\mathcal H}_i=\left\{
h: {\mathcal P}\rightarrow A_i
\right\}
\]
as the collection of all weak hypotheses with labels from $A_i$.
Clearly, ${\mathcal H}_i\subset {\mathcal H}$.

\begin{definition} ${\mathcal H}$ is iteratively weakly learnable
if $\forall A_i\subset A$ with $\abs{A_i}\ge 2$, $\exists \rho_i>0$ such that
for every distribution $d$ over observations $p \in {\mathcal P}$ and labels $y(p)$, 
\begin{equation} 
\min_{d \in \Delta({\mathcal P})} \max_{h \in \mathcal{H}_i} \sum_{p \in {\mathcal P}} \left(\mathbb{I}[h(p)=y(p)] -\mathbb{I}[h(p) \neq y(p)] \right) d(p) \geq \frac{2-\abs{A_i}}{\abs{A_i}} + \rho_i.
\label{eq: iterative weak learnability}
\end{equation}
\end{definition}

Note that $\frac{2-\abs{A_i}}{\abs{A_i}}$ is a decreasing function of $\abs{A_i}$. As $\abs{A_i}$ decreases, the requirement for weak learnability of ${\mathcal H}_i$ becomes more stringent, as the minimum performance criterion (the right hand side of \eqref{eq: iterative weak learnability} increases.

In Example \ref{ex: MS13}, ${\mathcal H}$ is weakly learnable, but not iteratively weakly learnable. To see this, note that if $A_i=\{1,2\}\subset A=\{1,2,3\}$, then
${\mathcal H}={\mathcal H}_i$. 
\[
\min_{0\le d(a)\le 1}\max \left[ 1-2d(a), -1+2d(a) \right]=0.
\]
However, the weak learnability over $A_i$ requires that the minimum performance be {\em strictly larger} than 0.

\subsection{Sufficiently Rich ${\mathcal H}$}

A fundamental question is whether we can find a class of simple hypothesis which is weakly learnable. With the set of the simplest classifiers, we can do better than the random guess over any subset of $A$ with at least two elements.

Suppose that $A$ and ${\mathcal P}$ are subsets of a vector space or a finite dimensional Euclidean space.  Consider $h$, which classifies ${\mathcal P}$ according to the value of a linear function, thus called a linear classifier.   Let ${\sf H}$ be a hyperplane in ${\mathbb R}^n$: $\exists\lambda\in {\mathbb R}^n$ and $\omega\in {\mathbb R}$ such that
\[
{\sf H}=\left\{ p\in {\mathbb R}^n \ | \ \lambda p =\omega \right\}.
\]
Define ${\sf H}_+$ as the close half space above ${\sf H}$:
\[
{\sf H}_+=\left\{ p\in {\mathbb R}^n \ | \ \lambda p \ge\omega \right\}.
\]

\begin{definition}
A single threshold (linear) classifier is a mapping
\[
h : {\mathcal P}\rightarrow A
\]
where $\exists a_+, a_-\in A$ such that
\[
h(p)=\begin{cases}
a^+ & \text{if } \ p\in {\sf H}_+ \\
a^- & \text{if } \ p\not\in {\sf H}_+.
\end{cases}
\]
\end{definition}

If $h$ is a single threshold classifier, $\abs{h(\mathcal{P})}=2$, which may be strictly smaller than $\abs{A}$. On the other hand, the set of single threshold classifiers is \emph{closed under permutations}: if $\pi: A \rightarrow A$ is a bijection, then $\pi(h(p)) \in \mathcal{H}$ for all $\pi$.  Being closed under permutations turns out to be the key property to obtain weak learnability.

\begin{lemma} 
Consider an environment where the image of $y(p)$ has $\abs{A'}$ elements where $A'\subset A$. Any hypothesis class that is closed under permutations can do as well as $1/\abs{A'}$ random guessing.
\end{lemma}

\begin{proof}
See Lemma A.1 in \citeasnoun{ChoandLibgober2020}.
\end{proof}

We use this result to show that the set of single-threshold classifiers is iteratively weakly learnable.  

\begin{corollary}  \label{th: binary classifier sufficient}  
The set of single threshold classifiers is iteratively weakly learnable.
\end{corollary}

\begin{proof}
  $\forall A_i\subset A$, the set of single threshold classifiers
  \[
    {\mathcal H}_i=\left\{
h : {\mathcal P}\rightarrow A_i 
\right\}
  \]
is closed under permutations, from which the conclusion follows.  
\end{proof}

In Example \ref{ex: MS13}, using single-threshold classifiers ensures that an extreme point of the convex hull of $\mathcal{P}$ can always be labeled correctly, and the rest can at least do as well as a random guess.  
The analogous result is known to hold for the case of $\abs{A}=2$.  Interestingly, the critical property in generalizing to the multi-class case is not changing the set of classifiers but instead allowing the labels to be sufficiently rich.

\section{Construction of Algorithm}
\label{Construction of Algorithm}

The algorithm, which we call $\tau^K_S$, iterates over two loops: within an epoch (inner loop) and across epochs (outer loop). $K$ is the minimum length of an epoch, which we increase to improve the accuracy of the final hypothesis of the algorithm. The size of each epoch is a random variable $K_i\ge K$. The algorithm has $I$ number of epochs, which is also a random variable. While the total number of periods to generate the final hypothesis is random variable $\sum_{i=1}^I K_i$, we show that as $K$ becomes large, the total number of periods is bounded by $\min (\abs{A},\abs{P})K$ with a probability close to 1.

\subsection{Preliminaries}

The algorithm is run in discrete period $k=1,2,\ldots$ within epoch $i=1,2,\ldots$. Each epoch lasts at least $K$ periods, but the terminal round of an epoch is random. At the start of an epoch, we set the period to $k=1$ and re-start the procedure.

It is convenient to enumerate the elements in $A=\{1,\ldots,\abs{A}\}$ as the first $\abs{A}$ positive integers, highlighting the fact that the weak learnability has little to do with the name of the labels but relies solely on the ordinal properties of $A$.
Let $\pi : A \rightarrow A$ be a permutation over $A$:
\[
\pi(j)=i_j\in\{1,\ldots,\abs{A}\} \qquad\forall j.
\]
Let $A_i=\{1,\ldots,\abs{A_i}\}\subset A$ be the collection of the first $\abs{A_i}$ positive integers which is the set of feasible labels at the beginning of epoch $i$ with $A_1=A$. Let $\pi_i$ be a permutation over $A_i$.

\subsection{Within Epoch $i$}
\label{Within Epoch i}

At the beginning of epoch $i$, $A_i\subset A$ and $y_i: {\mathcal P}\rightarrow A_i$ are given. Set $k=1$. $d_1(p)$ as the uniform distribution over $\mathcal{P}$.
The output of each stage consists three key components: an artificial probability distribution $d_k(p)$ over ${\mathcal{P}}$, a threshold rule $h_k$ in step $k$, and a positive weight $\alpha_{k}$. 

Suppose that $d_k(p)$ is defined $\forall p\in {\mathcal{P}}$. Choose $h_k$ by solving
\begin{equation}
\max _{h\in {\mathcal H}_i}\sum_{p\in {\mathcal{P}}} \left(\mathbb{I}[h(p)=y_i(p)]-\mathbb{I}[h(p) \neq y_i(p)] \right)d_k(p).
\label{eq: optimal h}
\end{equation}
We call $h_k$ an optimal classifier. Define
\begin{equation}
\epsilon_k=\Prob_{d_k}\left( h_k(p) \neq y_i(p) \right)
\label{eq: epsilon}
\end{equation}
as the probability that the optimal classifier $h_k$ at $k$ round misclassifies $p$ under $d_k$. If $\epsilon_k=0$, then we stop the training and output $h_k$ as the final hypothesis, which perfectly forecasts $y_i(p)$. 

Suppose that $\epsilon_k>0$. Define
\begin{equation}
\alpha_k =\frac{1}{2(\abs{A_i}-1)} \left[ \log\frac{(\abs{A_i}-1)(1-\epsilon_k)}{\epsilon_k}\right].
\label{eq: alpha 2}
\end{equation}
Note that the only difference between \eqref{eq: alpha 2} and \eqref{eq: SAMME alpha} is the coefficient $\frac{1}{2(\abs{A_i}-1)}$.  If $\abs{A_i}=2$, our algorithm is reduced to AdaBoost.

Define for each $p$, and each pair $(p,y(p))$,
\[
d_{k+1}(p)=\frac{
d_k(p)\exp \left(
-\alpha_k \left[ (\abs{A_i}-1) {\mathbb{I}}(h_k(p)= y_i(p)) -{\mathbb{I}}(h_k(p)\ne y_i(p))\right] \right)
}{Z_k}
\]
where
\[
Z_k=\sum_{p \in {\mathcal{P}}} d_k(p)\exp \left(
-\alpha_k\left[ (\abs{A_i}-1){\mathbb{I}}(h_k(p)=y_i(p)) - {\mathbb{I}}(h_k(p)\ne y_i(p)))\right] \right).
\]
Given $d_{k+1}$, we can recursively define $h_{k+1}$ and $\epsilon_{k+1}$.

Define $\forall a\in A_i=\{1,\ldots,\abs{A_i}\}$,
\[
F^k_a(p)=\sum_{s=1}^k\alpha_s {\mathbb I}(h_s(p)=a)
\]
and
\[
\Psi^k_a(p) =(\abs{A_i}-1) F^k_a(p) -\sum_{a'\ne a}F^k_{a'}(p).
\]
Note that
\begin{eqnarray}
\Psi^k_a(p)
& = & (\abs{A_i}-1) F^k_a(p) -\sum_{a'\ne a}F^k_{a'}(p) \nonumber \\
& = & \sum_{s=1}^k\alpha_s \left[ (\abs{A_i}-1) {\mathbb I}(h_s(p)=a) -\sum_{a'\ne a} {\mathbb I}(h_s(p)=a')\right] \nonumber \\
& = & \sum_{s=1}^k\alpha_s \left[ \abs{A_i} {\mathbb I}(h_s(p)=a) -\left( {\mathbb I}(h_s(p)=a)+\sum_{a'\ne a}{\mathbb I}(h_s(p)=a')\right)\right] \nonumber \\
& = & \sum_{s=1}^k\alpha_s \left( \abs{A_i} {\mathbb I}(h_s(p)=a) -1 \right).
\label{eq: alternative form}
\end{eqnarray}
Therefore, $\forall p$,
\begin{equation}
  a\in\arg\max_{A_i} F^k_a(p)
  \label{eq: max F}
\end{equation}
if and only if
\begin{equation}
a\in\arg\max_{A_i} \Psi^k_a(p).
\label{eq: bottom}
\end{equation}
We use $\Psi^k_a(p)$ to predict $y_i(p)$, instead of $F^k_a(p)$.   Since $\Psi^k_a(p)$ is a linear combination of ${\mathbb I}(h_s(p)=a)$, our algorithm is an ensemble algorithm.   Since $\Psi^k_a(p)$ is calculated recursively within an epoch, and the change between two epochs depends only on the final outcome from the previous epoch, we regard our algorithm a recursive algorithm.

Except for the case of $\abs{A_i}=2$, \eqref{eq: max F} does not guarantee that $a=y_i(p)$ with probability close to 1, as $k\rightarrow\infty$.   \citeasnoun{MukherjeeSchapire2013} argues that the failure is due to the fact that the weak learnability \eqref{eq: weak learnability} is too weak, and proposed a stronger criterion of EOR (Edge over Random) criterion.

We take a different tack.   We strengthen \eqref{eq: weak learnability} to
\eqref{eq: iterative weak learnability} by requiring the same condition for any subset of $A$ with at least two elements.  Instead of solving \eqref{eq: bottom}, we eliminate $a\in A_i$ which is not an element of $\arg\max_{A_i} \Psi^k_a(p)$.    A critical step is to show that $a=y_i(p)\in A_i$ survives the elimination process with a probability close to 1. Since $\abs{A}<\infty$, we can ``solve'' \eqref{eq: bottom} in a finite number of iterations.   Roughly speaking, 
\eqref{eq: weak learnability} may be too weak to solve \eqref{eq: max F} directly, but is sufficiently strong to solve the equivalent problem \eqref{eq: bottom} indirectly, if \eqref{eq: weak learnability} holds for any subset with at least two elements.

Recall that $K$ is the minimum number of periods in epoch $i$, which is a parameter of the algorithm.   Define
\[
K_i =\min\left\{
k\ge K \mid \forall p\in {\mathcal P}, \exists a\in A_i, \Psi^k_a(p) < 0 \right\}
\]
as the first period in epoch $i$ when each $p$ has a label $a\in A_i$ where $\Psi^k_a(p)<0$. Epoch $i$ terminates at $K_i$ period.

Before we start epoch $i+1$, we need to update $A_i$ and $y_i$. Roughly speaking, $A_{i+1}$ is obtained by eliminating $a\in A_i$ with $\Psi^{K_i}_a(p)<0$. Because the label with $\Psi^{K_i}_a(p)<0$ can vary across different $p$, we need additional work. First, we need to eliminate the equal number of labels from $A_i$ for each $p$, and then, permute the remaining elements so that we can ``homogenize'' the set of labels for each $p$.

Define
\[
B_i(p)=\{ a\in A_i \mid \Psi^{K_i}_a(p) <0 \}
\]
as the set of labels in $A_i$ with a negative value of $\Psi^{K_i}_a(p)$, which are the candidates of elimination.    Because $\abs{B_i(p)}$ can vary across $p$, we cannot eliminate all elements in $B_i(p)$ from $A_i$, if we want to keep the number of surviving elements in $A_i$ the same across $p$.    
Define
\[
N_i =\min_p \abs{B_i(p)}\ge 1
\]
which is at least 1,
by the definition of $K_i$. Recall that each element in $B_i(p)$ is a positive integer.
Instead of eliminating all elements in $B_i(p)$, we eliminate $N_i$ elements from $B_i(p)$ $\forall p\in {\mathcal P}$.     Recall that
\[
A_i=\{ 1,\ldots,\abs{A_i}\}.
\]
We can write
\[
B_i(p)=\{ i(p)_1,\ldots,i(p)_{\abs{B_i(p)}}\}.
\]
For convenience, we select the last $N_i$ elements in $B_i(p)$ to construct ${\underline B}_i(p)$:
\[
{\underline B}_i(p)=\left\{
i(p)_{\abs{B_i(p)}-N_i+1},\ldots, i(p)_{\abs{B_i(p)}}
  \right\}.
\]
By the construction of ${\underline B}_i(p)$, 
\begin{equation}
\abs{A_i\setminus {\underline B}_i(p)}=\abs{A_i\setminus {\underline B}_i(p')}
\qquad\forall p\ne p'\in {\mathcal P}.
\label{eq: well defined}
\end{equation}
The remaining step is to re-label each element in $A_i\setminus {\underline B}_i(p)$ so that each $p$ has the same set of labels.

Define
\[
A_{i+1}=\{ 1,\ldots, \abs{A_i\setminus {\underline B}_i(p)} \}
\]
as the collection of the first $\abs{A_i\setminus {\underline B}_i(p)}$ positive integers.
Note 
\[
  \abs{A_{i+1}}=\abs{A_i\setminus {\underline B}_i(p)} \qquad\forall p\in {\mathcal P}
\]
by \eqref{eq: well defined}.

We arrange the elements in $A_i\setminus {\underline B}_i(p)$ according to the index's order.  For each $p\in {\mathcal P}$, define a permutation
\[
\pi^p_i : A_{i+1}\rightarrow A_i\setminus {\underline B}_i(p)
\]
where
\[
\pi^p_i(1)=\min\{ a \in A_i\setminus {\underline B}_i(p) \}
\]
and let $\pi^p_i(1)=a_1$. Given $a_1,\ldots,a_{j-1}\in A_i\setminus {\underline B}_i(p) $,
\[
\pi^p_i(j)=\min\left\{ a \in A_i\setminus \left[ {\underline B}_i(p)\cup\{a_1,\ldots,a_{j-1}\} \right] \right\}.
\]
Since $\pi^p_i$ is a bijection over $A_{i+1}$, $\left(\pi^p_i\right)^{-1}$ exists.
The correct label in epoch $i$ is
\[
y_{i+1}(p)=\left(\pi^p_i\right)^{-1}(y_i(p)) \qquad\forall p.
\]

\subsection{Iteration over epoch}

Replace $A_i$ and $y_i$ by $A_{i+1}$ and $y_{i+1}$. Start the entire process described in Section \ref{Within Epoch i}. The iteration over epoch stops at the first $i$ where
\begin{equation}
\abs{A_i\setminus {\underline B}_i(p) }=1 \qquad\forall p
\label{eq: terminal round}
\end{equation}
Define
\[
I=\min\left\{ i \mid \abs{A_i\setminus {\underline B}_i(p) }=1 \qquad\forall p \right\}.
\]
as the number of epoch when the algorithm stops. Since epoch $i$ last $K_i$ periods, the total number of periods to train the algorithm is
\[
t=\sum_{i=1}^I K_i
\]
which is a random variable.

Let us call the constructed algorithm $\tau^K_S$ because the algorithm is parameterized by the minimum length of each epoch $K$. 
The element in $A_I\setminus {\underline B}_I(p)$ is the source of the final hypothesis of $\tau^K_S$. Since we know $\pi^p_i$ $\forall p\in {\mathcal P}$ and $\forall i\in\{1,\ldots,I\}$, the final hypothesis of $\tau^K_S$ is
\[
\tau^K_S(D_t)(p)=  \pi_1^p\left( \cdots \pi_I^p\left(A_I\setminus {\underline B}_I(p)\right)\right) \qquad\forall p\in {\mathcal P}.
\]
We are interested in $\Prob\left(\tau^K_S(D_t)(p)=y(p)\right)$.

\begin{theorem} \label{th: main result}
Assume that $\abs{A}<\infty$ and $\abs{{\mathcal P}}<\infty$, and that ${\mathcal H}$ is iteratively weakly learnable.
\begin{enumerate}
\item $I\le \min (\abs{A},\abs{P})$  
\item $\exists {\underline\rho}>0$ and ${\overline K}$ such that $\forall K\ge {\overline K}$, with probability at least $1-e^{-{\underline\rho}K}$,
\begin{enumerate}
\item $K_i=K$ $\forall i\in\{1,\ldots,I\}$
\item 
$\Prob\left( \tau_S^K(D_t)= y(p)\right)\ge 1-e^{-{\underline\rho} K}$ where $t=\sum_{i=1}^IK_i$. 
\end{enumerate}
\end{enumerate}
\end{theorem}

\begin{proof}
See Appendix \ref{Proof of Theorem th: main result}.
\end{proof}

\section{Generalization Error}
\label{Generalization Error}

Our last result shows how our approach lends itself to a simple formulation of the generalization error. Let ${\mathcal S}$ be the sample space generated by $K$ samples and $\Prob_{\mathcal S}$ be the probability measure over the sample space.   Theorem \ref{th: main result} proved that the training error vanishes at an exponential rate:
\[
  \Prob_{\mathcal S}\left(\tau^K_S(D_{KI})(p)\ne y(p)\right) \le e^{-{\underline\rho}K}
\]
where $K$ is the length of an epoch and $I$ is the number of epochs.   We know that
\[
I \le \min(\abs{A},\abs{P}).
\]

We can show that the generalization error of $\tau^K_S$ vanishes as $K\rightarrow\infty$.    Let $\Prob_{\mathcal D}$ be the population probability distribution.

\begin{theorem} \label{th: generalization error} $\forall m\ge K$,
\[
\Prob_{\mathcal D}\left(
\tau^K_S(D_{KI})(p)\ne y(p)
    \right) \le  \Prob_S\left(
\tau^K_S(D_{KI})(p)\ne y(p)
    \right)  +\bigO \left( \frac{1}{\sqrt{m}}\right).
  \]
\end{theorem}

\begin{proof}
See Appendix \ref{Proof of Theorem th: generalization error}.
\end{proof}

\section{Conclusion}
\label{Conclusion}

As $\tau^K_S$ is built on the framework of SAMME, the algorithm inherits the same statistical foundation as SAMME and is easier to motivate than other extensions of AdaBoost.     Some extensions of AdaBoost (e.g.,
\citeasnoun{MukherjeeSchapire2013}) may have to ``tune parameters'' of the algorithm to achieve efficiency.   The selection of a proper parameter of the algorithm requires information about ${\mathcal H}$.   In contrast, 
the construction of $\tau^K_S$ requires no knowledge about ${\mathcal H}$, other than the fact that the set of weak hypothesis satisfies iterative weak learnability.

It is straightforward to check for hypothesis classes that are closed under permutations, making it easy to verify the iterative weak learnability.  This allowed us to exhibit the algorithm straightforwardly, without needing to check that the algorithm can outperform random guessing for a rich set of possible cost functions.

\newpage
\appendix
\footnotesize

\section{Proof of Theorem \ref{th: main result}}
\label{Proof of Theorem th: main result}

\subsection{Uniform Bound for $I$}

We show
\[
I \le \min( \abs{A},\abs{{\mathcal P}} ).
\]
Recall that $\forall d_k$, $h_k$ is maximizing
\[
\sum_{p\in {\mathcal{P}}} \left(\mathbb{I}[h(p) = y_i(p)]-\mathbb{I}[h(p) \neq y_i(p)] \right)d_k(p).
\]
If $\exists p\in{\mathcal P}$ such that $h(p)=a\in A\setminus y\left( {\mathcal P}\right)$, then $h$ cannot be an optimal weak hypothesis.
Suppose that $a=a^+\in A\setminus y\left( {\mathcal P}\right)$.    
The decision maker can increase the performance by choosing $a^+$ from $y_i({\mathcal P})$.   Thus, $\forall k$, $h_k({\mathcal P})\subset y_i({\mathcal P})$.

We have $\forall k$, $\forall a\not\in y_i({\mathcal P})\subset A$,
\[
\Psi^k_a(p)=\sum_{s=1}^k \alpha_s\left( \abs{A} {\mathbb I}(h_s(p)=a ) -1\right)
=-\sum_{s=1}^k \alpha_s<0,
\]
since $h_s(p)\ne a$ $\forall a\in A\setminus y\left( {\mathcal P}\right)$. Every $a\in A\setminus y\left( {\mathcal P}\right)$ is eliminated at the end of the first epoch. Since $\abs{y\left( {\mathcal P}\right)}\le \abs{{\mathcal P}}$, the number of remaining elements after the first epoch is no more than $\min( \abs{A},\abs{{\mathcal P}} )$.

\subsection{Within Epoch}

We show that if the minimum length $K$ of an epoch is sufficiently large, then $K_i=K$ with a probability close to 1. That is, by the time when the algorithm reaches the scheduled end of $i$th epoch, $\exists a_i(p)\in A_i$ such that $\Psi^K_{a_i(p)}(p)<0$.
We calculate the lower bound of the probability. There exists $\rho_i>0$ such that the probability is bounded from below by $1-e^{-\rho_i K}$.

Recall that $\forall a\in A_i$,
\[
F^k_a(p)=\sum_{s=1}^k\alpha_s {\mathbb I}(h_s(p)=a)
\]
and
\begin{eqnarray*}
\Psi^k_a(p) &=& (\abs{A_i}-1) F^k_a(p) -\sum_{a'\ne a}F^k_{a'}(p) \\
        &=& \sum_{s=1}^k\alpha_s \left( \abs{A_i} {\mathbb I}(h_s(p)=a) -1 \right).
\end{eqnarray*}
Thus,
\begin{eqnarray*}
\sum_a \Psi^k_a(p) &=& \sum_a
\left( \sum_{s=1}^k \alpha_s \left[ \abs{A_i}{\mathbb I}(h_s(p)=a) -1 \right]\right) \\
&=& \sum_{s=1}^k\alpha_s \sum_a\left[ \abs{A_i}{\mathbb I}(h_s(p)=a) -1 \right] \\
&=& \sum_{s=1}^k\alpha_s \left[ \abs{A_i}\sum_a{\mathbb I}(h_s(p)=a) -\sum_a 1 \right] \\
&=& \sum_{s=1}^k\alpha_s \left[ \abs{A_i} -\abs{A_i}\right] =0.
\end{eqnarray*}

Iterative weak learnability implies that $\forall d_k$, $\exists h_k\in {\mathcal H}_i$ so that
\[
\sum_{p\in {\mathcal{P}}} \left(\mathbb{I}[h_k(p) = y_i(p)]-\mathbb{I}[h_k(p) \neq y_i(p)] \right)d(p) \ge \frac{2-\abs{A_i}}{\abs{A_i}}+\gamma_k.
\]
for some $\gamma_k>0$. The weak learnability implies that $\inf_k\gamma_k>0$.

\begin{lemma} \label{lm: key result}
\[
\Prob \left( \Psi_{y_i(p)}^k(p) \le 0 \right) \le e^{-\rho_i k}.
\]
\end{lemma}

\begin{proof}
Note that
\begin{eqnarray}
\Prob \left( \Psi_{y_i(p)}^k(p) \le 0 \right) & =&
\Expect_{d_1}{\mathbb I}\left(\Psi^k_{y_i(p)}(p) \le 0\right) \nonumber \\
&= & \Expect_{d_1}{\mathbb I}\left(-\Psi^k_{y_i(p)}(p) \ge 0\right) \nonumber \\
& \le & \Expect_{d_1}\exp\left(-\Psi^k_{y_i(p)}(p)\right). \label{eq: last line}
\end{eqnarray}
From the updating formula of $d_t$,
\begin{eqnarray}
d_{k+1}(p)
& = & \frac{ d_k(p)\exp \left[-\alpha_k(\abs{A_i}-1){\mathbb I}(h_k(p)=y_i(p))+\alpha_k{\mathbb I} (h_k(p)\ne y_i(p)) \right]}{Z_k} \nonumber \\
& = & \frac{ d_1(p)\exp \left[-\sum_{s=1}^k\alpha_k(\abs{A_i}-1){\mathbb I}(h_s(p)=y_i(p))+\alpha_k{\mathbb I} (h_k(p)\ne y_i(p)) \right]}{\prod_{s=1}^kZ_s} \nonumber \\
& = & \frac{d_1(p) \exp \left[ -(\abs{A_i}-1)F^k_{y_i(p)}(p)
+\sum_{a\ne y_i(p)}F^k_a(p)\right]}{\prod_{s=1}^kZ_s} \\
& = & \frac{d_1(p) \exp \left[ -\Psi^k_{y_i(p)}(p)\right]}{\prod_{s=1}^kZ_s}.
\label{eq: another line}
\end{eqnarray}
By summing over $p$, we have
\[
\sum_p d_1(p) \exp \left[ -\Psi^k_{y_i(p)}(p)\right]=\prod_{s=1}^kZ_s.
\]
Note $\forall k\ge 1$, $\forall s\{1,\ldots,k\}$,
\begin{eqnarray*}
Z_s & = & \sum_p d_s(p)\exp \left[-\alpha_s(\abs{A_i}-1) {\mathbb I}(h_s(p)=y_i(p))+\alpha_s{\mathbb I} (h_s(p)\ne y_i(p)) \right] \\
& = & \sum_{h_s(p)=y_i(p)}d_s(p) e^{-(\abs{A_i}-1)\alpha_s}
+\sum_{h_s(p)\ne y_i(p)}d_s(p) e^{\alpha_s} \\
& = & e^{-\alpha_s (\abs{A_i}-1)}(1-\epsilon_s)+ e^{\alpha_s} \epsilon_s \\
& = & ( \abs{A_i}-1)^{-\frac{1}{2}}\epsilon_s^{\frac{1}{2}}(1-\epsilon_s)^{\frac{1}{2}}
      +(\abs{A_i}-1)^{\frac{1}{2(\abs{A_i}-1)}}(1-\epsilon_s)^{\frac{1}{2(\abs{A_i}-1)}}\epsilon_s^{1-\frac{1}{2(\abs{A_i}-1)}} \\
  &\equiv & \varphi(\epsilon_s).
\end{eqnarray*}

\begin{lemma}
$\varphi(\epsilon)$ is a strictly concave function of $\epsilon\in [0,1]$.
\begin{equation}
\varphi \left( 1-\frac{1}{\abs{A_i}}\right)=1
\label{eq: value}
\end{equation}
and
\begin{equation}
\varphi' \left( 1-\frac{1}{\abs{A_i}}\right)=0.
\label{eq: derivative}
\end{equation}
\end{lemma}

\begin{proof}
$\epsilon^{\frac{1}{2}}(1-\epsilon)^{\frac{1}{2}}$ and $(1-\epsilon)^{\frac{1}{2(\abs{A_i}-1)}}\epsilon^{1-\frac{1}{2(\abs{A_i}-1)}}$ are strictly concave functions over $\epsilon$.  $\varphi(\epsilon)$ is the sum of two strictly concave functions.   Thus, $\varphi(\epsilon)$ is strictly concave.

\eqref{eq: value} follows from a simple calculation.
\begin{eqnarray*}
&&  \varphi\left(1-\frac{1}{\abs{A_i}}\right) \\
& = &   (\abs{A_i}-1)^{-\frac{1}{2}}\left( \frac{\abs{A_i}-1}{\abs{A_i}}\right)^{\frac{1}{2}}\left( \frac{1}{\abs{A_i}}\right)^{\frac{1}{2}}
+(\abs{A_i}-1)^{\frac{1}{2(\abs{A_i}-1)}}\left( \frac{\abs{A_i}-1}{\abs{A_i}}\right)^{1-\frac{1}{2(\abs{A_i}-1)}}\left( \frac{1}{\abs{A_i}}\right)^{\frac{1}{2(\abs{A_i}-1)}}  \\      
& = & \frac{1}{\abs{A_i}}+\left(1-\frac{1}{\abs{A_i}}\right) =1.
\end{eqnarray*}

The calculation of $\varphi'$ is a little more involved.   
\begin{eqnarray*}
&&  \varphi'\left(1-\frac{1}{\abs{A_i}}\right) \\
  & = &   (\abs{A_i}-1)^{-\frac{1}{2}}\left[ \frac{1}{2} \left(\frac{1}{\abs{A_i}-1}\right)^{\frac{1}{2}}
        -\frac{1}{2}\left( \abs{A_i}-1 \right)^{\frac{1}{2}} \right]   \\
  && \ \    +  (\abs{A_i}-1)^{\frac{1}{2(\abs{A_i}-1)}}\left[ \left(1-\frac{1}{2(\abs{A_i}-1)}\right) 
    \left(\frac{1}{\abs{A_i}-1}\right)^{\frac{1}{2(\abs{A_i}-1)}}
-\frac{1}{2(\abs{A_i}-1)}\left( \abs{A_i}-1 \right)^{1-\frac{1}{2(\abs{A_i}-1)}} \right] \\
  & = & \left[ \frac{1}{2}\frac{\abs{A_i}}{\abs{A_i}-1} -1 \right]
        +\left[1- \frac{1}{2}\frac{\abs{A_i}}{\abs{A_i}-1}  \right] \\
& = & 0.
\end{eqnarray*}
\end{proof}

Since ${\mathcal H}_i$ is weakly learnable, $\exists \gamma_k>0$ such that
\[
0 \le \epsilon_k \le 1-\frac{1}{\abs{A_i}}-\gamma_k < 1-\frac{1}{\abs{A_i}}.
\]
Weak learnability implies that 
\[
\gamma_k>0 \ \ \text{and} \ \ \gamma =\inf_k \gamma_k >0.
\]
Since $\varphi$ is strictly concave, $\varphi'(\epsilon)>0$ $\forall\epsilon <1-\frac{1}{\abs{A_i}}$.   Thus, $\exists {\hat\gamma}>0$ such that $\forall\epsilon_k \le 1-\frac{1}{\abs{A_i}}-\gamma$,
\[
\varphi(\epsilon_k) \le\varphi\left(1-\frac{1}{\abs{A_i}}-\gamma\right) =1-{\hat\gamma}<1.
\]
Thus,
\[
Z_k = \varphi(\gamma_k) \le \varphi (\gamma) \le 1-{\hat\gamma}
\]
for some ${\hat\gamma}>0$. Thus,
\[
d_{k+1}(p)=d_1(p)\prod_{s=1}^kZ_k \le 
\frac{1}{\abs{{\mathcal P}}} (1-{\hat\gamma})^k = \frac{1}{\abs{{\mathcal P}}} e^{-\rho_i t}
\]
where
\begin{equation}
\rho_i =-\log (1-{\hat\gamma}) >0.
\label{eq: rho definition}
\end{equation}
\end{proof}

By the definition, $\forall p$, $\forall k$,
\begin{equation}
\sum_{a\in A_i}\Psi^k_a(p) =0.
\label{eq: summation}
\end{equation}
By \eqref{eq: summation}, $\forall p$, if 
$\Psi^K_{y(p)}(p) >0$, then $\exists {\underline a}_i(p)\in A_i$ such that
\[
\Psi^K_{{\underline a}_i(p)}(p) <0.
\]
Lemma \ref{lm: key result} implies 
\[
{\underline a}_1(p)\ne y_i(p)
\]
with probably at least $1-e^{-\rho_i K}$. Hence, $K_i=K$ with probability at least $1-e^{-\rho_i K}$.

As Lemma \ref{lm: key result} indicates, if $a=y(p)$, then  $\Psi^k_a(p)>0$ almost surely as $K\rightarrow\infty$.  If $\abs{A}=2$, then \eqref{eq: summation} implies that exactly one label is associated with a positive value of $\Psi^k_a(p)$.   Thus, searching for $a$ with $\Psi^k_a(p)>0$ is the proper way to identify a correct label.   However, if $\abs{A}>2$, then the positive sign of  $\Psi^k_a(p)$ does not imply that $a=y(p)$.     However, if $\Psi^k_a(p)<0$, then Lemma \ref{lm: key result} implies that $a\ne y(p)$ almost surely, as $K\rightarrow\infty$.  The idea of $\tau^K_S$ is to eliminate $a$ with $\Psi^k_a(p)<0$ in each epoch, until we have a single element left.

\subsection{Across Epochs}

We construct $A_{i+1}$ by eliminating some elements from $A_i$ and permuting the remaining elements according to $\pi_i$. With probability $1-e^{-\rho_i K}$, the true label survives the elimination from $A_i$, and is included in $A_{i+1}$, permuted according to $\pi_i$. 

Since we eliminate at least one element from $A_i$, the total number of epochs cannot be larger than $\abs{A}$. Let $I$ be the epoch when $A_{I+1}$ is a singleton. With probability
\[
\prod_{i=1}^{I} (1-e^{-\rho_i K}),
\]
the remaining label $y_I(p)$ is the real label which has been permuted by a series of $\pi_1,\ldots,\pi_I$. We can recover the real label
\[
y(p)=\pi_1\left(\cdots \pi_I(y_I(p))\right).
\]
Define
\[
{\underline\rho}=\frac{1}{2}\min \rho_i >0.
\]
Then, for a sufficiently large $K$,
\[
\prod_{i=1}^{I} (1-e^{-\rho_i K})\ge 1-\abs{A}e^{-K \min_i \rho_i}
=1-e^{-K(\min_i\rho_i+\frac{1}{K}\log\abs{A}) } \ge 1- e^{-{\underline\rho} K}
\]
which completes the proof.

\section{Proof of Theorem \ref{th: generalization error}}
\label{Proof of Theorem th: generalization error}

The proof follows the same logic as the proof of Theorem 5.1 on page 98 of \citeasnoun{SchapireandFreund12}.  We need to change the notation to accommodate multiple classes $(\abs{A}\ge 2)$.    For the reference, we sketch the proof by following the original proof of Theorem 5.1 in \citeasnoun{SchapireandFreund12}, while pointing out the place where we need a change and referring the rest to \citeasnoun{SchapireandFreund12}.

\subsection{Preliminaries}

Recall that $\forall i\le I$, $\forall a\in A_i$
\[
\Psi^k_a(p)=\sum_{s=1}^k\alpha_s\left[ \abs{A_i}{\mathbb I} (h_s(p)=a) -1\right].
\]
Define
\[
g^s_a(p)=\abs{A_i}{\mathbb I} (h_s(p)=a) -1,
\]
\[
\beta_s=\frac{\alpha_s}{\sum_{s'=1}^k\alpha_{s'}}
\]
and
\[
\psi^k_a(p)=\sum_{s=1}^k\beta_s g_a^s(p).
\]
We drop subscript $a$ because we focus on $a=y(p)$ and write instead
\[
g(p)=g_{y(p)}(p) \ \text{and} \ \psi^k(p)=\psi^k_{y(p)}(p).
\]
Recall that ${\mathcal H}$ is the collection of feasible weak hypotheses, whose generic element is $h$.
Define
\[
{\mathcal G}=\left\{ g \mid \exists h\in {\mathcal H}, \ 
g(p)=\abs{A_i}{\mathbb I} (h(p)= y(p)) -1 \right\}
\]
For a fixed $n$, define
\[
  {\mathcal A}_n =\left\{
f : p\mapsto \frac{1}{n}\sum_{j=1}^n g_j (p) \mid g_1,\ldots,g_n\in {\mathcal G}
    \right\}.
\]
Let ${\tilde g}_1,\ldots,{\tilde g}_n$ be the collection of $n$ {\em randomly} selected elements from $G$ and
\[
{\tilde f}=\frac{1}{n}\sum_{j=1}^n {\tilde g}_j (p) 
\]
where
\[
  \Expect g_j =\sum_{s=1}^k\beta_s g_{y(p)}^s(p)=\psi^k_{y(p)}(p).
\]

If
\[
\tau^K_S(D_{KI})(p)\ne y(p), 
\]
then $y(p)$ is eliminated from $A_i$ in $i$-th epoch $i <I$, which implies
\[
\psi^K_{y(p)}(p) <0.
\]
Thus,
\begin{eqnarray*}
&& \Prob_{\mathcal D}\left( \tau^K_S(D_{KI})\ne y \right) \\
& \le & \prod_{i=1}^I \Prob_{\mathcal D}\left( \exists i <I, \ 
\psi^K_{y(p)}(p) < 0 \mid  \psi^K_{y(p)}(p) \ge 0, j\le i   \right) \\
&= &  \prod_{i=1}^I \Prob_{\mathcal D}\left(     \psi^K_{y(p)}(p) < 0 \ \text{at the end of} \ i\text{-th epoch}\right).
\end{eqnarray*}
The last equality stems from the facts that the samples are independently drawn over periods, and that we re-initialize the algorithm at the beginning of each epoch.

We only show that in the first epoch
\[
  \Prob_{\mathcal D}\left(     \psi^K_{y(p)}(p) < 0\right)
  \le
    \Prob_{\mathcal S}\left(     \psi^K_{y(p)}(p) < 0\right) +\bigO\left(\frac{1}{\sqrt{m}}\right)
\]
where ${\mathcal S}$ is the collection of all $m$ samples. The proof for the remaining epochs follows the same logic.  Since $I\le \min\left(\abs{A},\abs{{\mathcal P}}\right)\le \abs{A}$, we then have
\[
  \Prob_{\mathcal D}\left( \tau^K_S(D_{KI})\ne y \right)
  \le   \Prob_{\mathcal S}\left( \tau^K_S(D_{KI})\ne y \right) + 
\abs{A}\bigO\left(\frac{1}{\sqrt{m}}\right).
\]

\subsection{Sketch of Proof}

\begin{lemma}  (Re-statement of Lemma 5.2 in \citeasnoun{SchapireandFreund12})
$\forall\theta$,
  \begin{equation}
    \Prob_{{\tilde f}}\left( | {\tilde f}(p)-f(p) | \ge \frac{\theta}{2} \right)
  \le 2 e^{-\frac{n\theta^2}{2\abs{A}^2}}\equiv \beta_{n,\theta}.
  \label{eq: scaling}
  \end{equation}
\end{lemma}

\begin{proof}
Note that the only change from the original statement in \citeasnoun{SchapireandFreund12} is $\beta_{n,\theta}$ which needs to be scaled to accommodate the case $\abs{A}>2$.    Note that $g(p)\in \{ \abs{A}-1,-1\}$.  We need to scale $g$ into a random variable in $[-1,1]$ and modify $\beta_{n,\theta}$ accordingly.  The remainder of the proof is identical with the original proof of Lemma 5.2 in \citeasnoun{SchapireandFreund12}.
\end{proof}

Lemma 5.3 of \citeasnoun{SchapireandFreund12} is not needed.   The proof of Lemma 5.4 of \citeasnoun{SchapireandFreund12} applies to our case without modifying the parameter, from which the conclusion of Theorem \ref{th: generalization error} follows.

\newpage
\normalsize
\bibliographystyle{econometrica}
\bibliography{adaboost}

\end{document}